\documentclass[final,breaklinks]{colt2020-plain}

\title[Sharper Bounds for Uniformly Stable Algorithms]{Sharper Bounds for Uniformly Stable Algorithms}
\usepackage{times}
\usepackage{amscd,ifthen}
\coltauthor{%
 \Name{Olivier Bousquet} \Email{obousquet@google.com}\\
 \addr Google Research, Brain Team, Z\"urich
 \AND
 \Name{Yegor Klochkov} \Email{yk376@cam.ac.uk}\\
 \addr Cambridge-INET, Faculty of Economics, University of Cambridge%
 \AND
 \Name{Nikita Zhivotovskiy} \Email{zhivotovskiy@google.com}\\
 \addr Google Research, Brain Team, Z\"urich%
}

\newcommand{\Ind}{\boldsymbol{1}}
\newcommand{\E}{\mathbb{E}}
\newcommand{\R}{\mathbb{R}}
\def\P{\mathbb{P}}

\def\Bc{{\breve{B}}}
\def\cond{\vert\;}

\begin{document}

\maketitle

\begin{abstract}%
Deriving generalization bounds for stable algorithms is a classical question in learning theory taking its roots in the early works by \cite{Vapnik74} and \cite{Rogers78}. In a series of recent breakthrough papers by \cite{Feldman2018, Feldman2019}, it was shown that the best known high probability upper bounds for uniformly stable learning algorithms due to \cite{Bousquet02} are sub-optimal in some natural regimes. To do so, they proved two generalization bounds that significantly outperform the simple generalization bound of~\cite{Bousquet02}. Feldman and Vondr\'ak also asked if it is possible to provide sharper bounds and prove corresponding high probability lower bounds. This paper is devoted to these questions: firstly, inspired by  the original arguments of~\cite{Feldman2019}, we provide a short proof of the moment bound that implies the generalization bound stronger than both recent results~\citep{Feldman2018, Feldman2019}. Secondly, we prove general lower bounds, showing that our moment bound is sharp (up to a logarithmic factor) unless some additional properties of the corresponding random variables are used. Our main probabilistic result is a general concentration inequality for weakly correlated random variables, which may be of independent interest.
\end{abstract}

%\begin{keywords}
%  uniformly stable algorithms, generalization bounds, high probability bounds 
%\end{keywords}

\section{Introduction}

The main motivation of our study is the analysis of {\em stable} learning algorithms (we recall the definition introduced in \citep{Bousquet02} below). We are given an i.i.d sample of points $S = \{(X_1, Y_1), \ldots, (X_n, Y_n)\}$ distributed independently according to some unknown measure $P$ on $\mathcal{X}\times \mathcal{Y}$. A learning algorithm $A:(\mathcal{X}\times \mathcal{Y})^n \to \mathcal{Y}^{\mathcal{X}}$ maps a training sample to a function mapping the instance space $\mathcal X$ into the space of labels $\mathcal Y$. The output of the learning algorithm based on the sample $S$ will be denoted by $A_S$. The quality of the function returned by the algorithm is measured using a loss function $\ell: \mathcal{Y} \times \mathcal{Y} \to \mathbb{R}_{+}$.
More precisely, given the random sample $S = \{(X_1, Y_1), \ldots, (X_n, Y_n)\}$, the \emph{risk} of the algorithm is defined as
\[
R(A_{S}) = \E_{(X, Y)\sim P}\, \ell(A_{S}(X), Y)\,.
\]
One of the fundamental questions in statistical learning is how to estimate the risk $R(A_S)$ of an algorithm from  the sample $S$ itself or with limited additional data.
A very good estimate can be obtained if one has access to additional data from the same distribution (a test set), and in practice the so-called test error is used as such. However, the question of whether one could obtain an accurate estimate from no or very limited additional data is essential if one wants to understand the statistical properties of the learning algorithm.

If one wishes to estimate $R(A_S)$ with no additional data, a natural quantity to consider is the so-called \emph{empirical risk} defined as
\[
R_{\text{emp}}(A_S) = \frac{1}{n}\sum\limits_{i = 1}^n\ell(A_{S}(X_i), Y_i)\,.
\]

A large body of work has been dedicated to obtaining \emph{generalization bounds}, i.e., high probability bounds on the error of the empirical risk estimator: $R(A_{S})  - R_{\text{emp}}(A_S)$. A standard way to prove the generalization bounds is based on the sensitivity of the algorithm to changes in the learning sample, such as leaving one of the data points out or replacing it with a different one. To the best of our knowledge, this idea was first used by Vapnik and Chervonenkis to prove the in-expectation generalization bound for what now is known as hard-margin Support Vector Machine \citep{Vapnik74}. 
Later works by Devroye and Wagner used a notion of stability to prove high probability generalization bounds for $k$-nearest neighbors \citep{Devroye79}. 
\cite{Bousquet02} provide an extensive analysis of different notions of stability and the corresponding (sometimes) high probability generalization bounds. In the context of the stochastic gradient method, uniform stability was shown by \cite*{Hardt2016}. 
Among some recent contributions on high probability upper bounds based on the notions of stability is the paper \citep{Maurer17}, which studies generalization bounds for a particular case of linear regression with a strongly convex regularizer, as well as the recent works~\citep{Zhivotovskiy17b, Bousquet20}, which provide sharp exponential upper bounds for the SVM in the realizable case. For additional background on the topic we refer to the recent papers \citep{Feldman2018,Feldman2019} and references therein.

Let us define stability more precisely. For the sake of simplicity, we denote $\mathcal{Z} = \mathcal{X}\times \mathcal{Y}$. The learning algorithm $A$ is \emph{uniformly stable} with parameter $\gamma$ if given any samples
\[
S = \{z_1, \ldots, z_n\} \in \mathcal{Z} ^n \quad\text{and}\quad S^{i} = \{z_1, \ldots,  z_{i - 1} , z_i^{\prime}, z_{i + 1}, \ldots, z_n\} \in \mathcal{Z} ^n,
\]   
for any $(x, y) \in \mathcal{X}\times \mathcal{Y}$, we have
\begin{equation}\label{stability_original}
|\ell(A_{S}(x), y) - \ell(A_{S^i}(x), y)| \le \gamma.
\end{equation}
Note that this can be thought of as a deterministic version of Differential Privacy, and there is active work exploring the connections between the two notions (see e.g., \citep*{Wang2016}).

Several works have focused on using the stability technique to derive bounds on the generalization error and we also consider this question here. This raises the question of the relevance of such an approach for studying modern machine learning models such as those used in Deep Learning. Indeed, it has been observed that very often the models can be trained to achieve zero empirical error which contradicts the possibility of having small generalization bounds. These models essentially perform {\em interpolation} or {\em memorization} of the data and the question of why this would not lead to overfitting is an active area of research \citep*{belkin2018overfitting}. It turns out that the stability technique can also be used in the interpolation regime as was actually initially done by \cite{Devroye79} when studying $k$-nearest neighbors (e.g., $1$-nearest neighbor rules do interpolate the data, yet the stability approach can provide meaningful bound). We will return to this point in Section \ref{discussion}.

Let us now recall the known results about generalization bounds for uniformly stable algorithms. In order to simplify the notation in what follows, we rescale the quantity of interest by $n$ and thus write the bounds for
\[
n(R(A_{S}) - R_{\text{emp}}(A_{S}))\,,
\]
and additionally assume that the loss function $\ell$ is bounded by $L$.

The basic and, until very recently, the best known result is the high probability upper bound in \citep{Bousquet02} which states that, with probability at least $1 - \delta$,
\begin{equation}
\label{bousquetbound}
n(R(A_{S}) - R_{\text{emp}}(A_{S})) \lesssim (n\sqrt{n}\gamma + L\sqrt{n})\sqrt{\log \left( \frac{1}{\delta}\right)}.
\end{equation}
It is easy to observe that this generalization bound is tight only when $\gamma \lesssim \frac{1}{n}$, which means that only under this assumption, the generalization error $R(A_{S}) - R_{\text{emp}}(A_{S})$ converges to zero with the optimal rate $\frac{1}{\sqrt{n}}$. However, in some applications the regime $\gamma \lesssim \frac{1}{\sqrt{n}}$ is of interest, and the bound \eqref{bousquetbound} can not guarantee any convergence. 

In a series of breakthrough papers, \cite{Feldman2018, Feldman2019} managed to correctly capture this phenomenon.
They first showed a generalization bound of the form
\begin{equation}
\label{firstfeldmanineq}
n(R(A_{S}) - R_{\text{emp}}(A_{S})) \lesssim (n\sqrt{\gamma L} + L\sqrt{n})\sqrt{\log \left( \frac{1}{\delta}\right)},
\end{equation}
where as before, the parameter $\gamma$ corresponds to the stability, and the parameter $L$ bounds the loss function $\ell$ uniformly.
In their second paper, \citeauthor{Feldman2019} show a stronger generalization bound
\begin{equation}
\label{secondfeldmanineq}
n(R(A_{S}) - R_{\text{emp}}(A_{S})) \lesssim n\gamma \log^2 n 
+ n\gamma\log n \log \left( \frac{1}{\delta}\right) + 
L\sqrt{n} \sqrt{\log \left( \frac{1}{\delta}\right)}.
\end{equation}
Up to logarithmic factors, the bound \eqref{secondfeldmanineq} shows that with high probability in the regime $\gamma \sim \frac{1}{\sqrt{n}}$, the generalization error $R(A_{S}) - R_{\text{emp}}(A_{S})$ converges to zero with the optimal rate $\frac{1}{\sqrt{n}}$. However, as claimed by Feldman and Vondr\'ak, the bound \eqref{firstfeldmanineq} should not be wholly discarded since it does not contain additional logarithmic factors $\log n$ and $(\log n)^2$. More importantly, the bound \eqref{firstfeldmanineq} is \emph{sub-gaussian}, which means that the dependence on $\delta$ comes only in the form $\sqrt{\log \left( \frac{1}{\delta}\right)}$. At the same time, the bound \eqref{secondfeldmanineq} shows both \emph{sub-gaussian} and the \emph{sub-exponential} regimes since it contains \emph{two types} of terms: $\sqrt{\log \left( \frac{1}{\delta}\right)}$ and $\log \left( \frac{1}{\delta}\right)$. We will discuss the notions of \emph{sub-gaussian} and \emph{sub-exponential} high probability upper bounds below.

In \citep{Feldman2019}, the authors ask if their high-probability upper bounds \eqref{firstfeldmanineq} and \eqref{secondfeldmanineq} can be strengthened and if they can be matched by a high probability lower bound. In this paper, we are making some progress in answering both questions. We shortly summarize our findings:
\begin{itemize}
	\item Our main probabilistic result is Theorem \ref{concentration_thm}, presented in Section~\ref{upper_bound}. As one of the immediate corollaries, it implies the risk bound of the form
	\begin{equation}
	\label{bestbound}
	n|R(A_{S}) - R_{\text{emp}}(A_{S})| \lesssim n\gamma\log n\log \left( \frac{1}{\delta}\right)+ L\sqrt{n}\sqrt{\log\left( \frac{1}{\delta}\right)},
	\end{equation}
	which removes the unnecessary term $n\gamma (\log n)^2$ from \eqref{secondfeldmanineq}.
	We emphasize that our analysis is inspired by the original sample-splitting argument of \citeauthor{Feldman2019}, although our proof is significantly more straightforward. In particular, we avoid several involved technical steps, which ultimately leads us to better generalization bounds. 
	\item Our Theorem \ref{concentration_thm} will also easily imply the sub-gaussian bound \eqref{firstfeldmanineq}, which was originally shown via the techniques taking their roots from Differential Privacy. Therefore, we also make a natural bridge between the bounds of the form \eqref{firstfeldmanineq} and \eqref{secondfeldmanineq}, which have different dependencies on $\log \frac{1}{\delta}$, $L$, and $\gamma$.
	\item In Section \ref{lower_bounds}, we show that the bound of our Theorem \ref{concentration_thm} is tight unless some additional properties of the corresponding random variables are used. Our lower bounds are presented by some specific functions satisfying the assumptions of Theorem \ref{concentration_thm}. We remark that our lower bound does not completely answer the question of the optimality of \eqref{bestbound} for uniformly stable algorithms, as it only shows the tightness of the bound implying \eqref{bestbound}. We discuss it in more detail in Section \ref{lower_bounds}.
\end{itemize} 

\section{Preliminaries}
\subsection{Notation} We provide some notation that will be used throughout the paper. 
The symbol $\Ind[A]$ will denote an indicator of the event $A$.
For a pair of non-negative functions $f, g$ the notation $f \lesssim g$ or $g \gtrsim f$ will mean that for some universal constant $c>0$ it holds that $f \le cg$. Similarly,  we introduce $f \sim g$ to be equivalent to $g \lesssim f  \lesssim g$. For $a, b \in \mathbb{R}$ we define $a \wedge b = \min\{a, b\}$ and $a \lor b = \max\{a, b\}$. The $L_p$ norm of a random variable will be denoted as $\|Y\|_p = (\E|Y|^p)^{\frac{1}{p}}$. Let $[k] $ denote the set $\{1, \ldots, k\}$. To avoid some technical problems for $x> 0$, by $\log x$ we usually mean $\log x \lor 1$

In what follows, we work with functions of \(n\) independent variables \( Z = (Z_1, \dots, Z_n) \). 
For \( A \subset [n] \) we will write \( Z_{A} = (Z_{j})_{j \in A} \). We also use the following notation
\[
Z^i = (Z_1, \ldots, Z_{i - 1}, Z_i^{\prime}, Z_{i + 1}, \ldots, Z_n),
\]
where $Z_i^{\prime}$ is an independent copy of $Z_i$.
In addition, for \( f = f(Z) \) and \(A \subset [n]\) we write \(  \| f \|_{p}(Z_A) = \E^{1/p} [|f|^p \cond Z_A] \). 
In particular, if we have an a.s. bound \( \| f \|_{p}(Z_A) \leq C \) for any realisation of \(Z_A\), then by a simple integration argument we have
\begin{equation}\label{moment_cond_to_uncond}
\| f \|_{p} = (\E \E[|f|^{p} \cond Z_A ])^{1/p} \leq C,
\end{equation}
i.e., in this sense, a conditional bound is stronger than the unconditional one. Finally, for $x = (x_1, \ldots, x_d) \in \mathbb{R}^d$ slightly abusing the notation we set $\|x\|_2 = \left(\sum\limits_{i = 1}^d x_i^{2}\right)^{1/2}$ and $\|x\|_{\infty} = \max\limits_{i \in [n]}|x_i|$.

\subsection{Equivalence of Tails and Moments}
Probabilistic bounds in Learning Theory are often of the form
\begin{equation}\label{deviation_form}
Y \leq a \sqrt{\log \left(\frac{1}{\delta}\right)} + b \log \left(\frac{1}{\delta}\right),
\end{equation}
with probability at least \(1 - \delta \) for any \( \delta \in (0, 1) \) and some $a, b \ge 0$.
Here, \(Y \) is a random variable of interest, e.g., the excess risk. The term with \( \sqrt{\log \left(\frac{1}{\delta}\right)} \) is referred to as a \emph{sub-gaussian} tail, as it matches the deviations of a Gaussian random variable. The term with \( \log \left(\frac{1}{\delta}\right) \) is called a \emph{sub-exponential} tail for a similar reason. In general, the bound above represents a \emph{mixture of sub-gaussian and sub-exponential} tails. In particular, all the known generalization bounds \eqref{bousquetbound}, \eqref{firstfeldmanineq}, \eqref{secondfeldmanineq} are of the form \eqref{deviation_form}.

An alternative way of studying tail bounds is via the moment norms. It is well-known that for a Gaussian random variable $Y$, there exists some $a\ge 0$ such that for any $p\ge 1$, $\| Y \|_{p} \leq \sqrt{p} a$, while when $Y$ is sub-exponential, we have $\| Y \|_{p} \leq pb$ for some $b\ge 0$ (see e.g., Propositions~{2.5.2} and {2.7.1} in \citep{Vershynin2016HDP}).
In what follows, we will consider the random variables with two levels of moments, that is for some $a, b \ge 0$ that do not depend on $p$,
\[
\|Y\|_p \le \sqrt{p}a + pb, \qquad \forall p \geq 1.
\]
In fact, the above bound and the bound \eqref{deviation_form} are equivalent up to a constant, as the following simple result suggests.

\begin{lemma}[Equivalence of tails and moments]
	\label{dev_moment_equiv:lemma}
	Suppose, a random variable has a mixture of sub-gaussian and sub-exponential tails, in the sense that it satisfies for any \(\delta \in (0, 1)\) with probability at least $1 - \delta$,
	\[
	|Y| \le a \sqrt{\log \left(\frac{e}{\delta}\right) } + b \log\left( \frac{e}{\delta}\right),
	\]
	for some \( a, b \geq 0 \). Then, for any \( p \geq 1\) it holds that
	\[
	\| Y \|_{p} \leq 3 \sqrt{p} a + 9 pb.
	\]
	
	And vice versa, if \( \| Y \|_{p} \leq \sqrt{p} a + pb \) for any \( p \geq 1\) then for any \( \delta \in (0, 1)\) we have, with probability at least $1 - \delta$,
	\[
	|Y| \le e\left(a \sqrt{\log \left(\frac{e}{\delta}\right) } + b \log\left( \frac{e}{\delta}\right)\right) .
	\]
\end{lemma}

The proof is a simple adaptation of Theorem~{2.3} from \citep{boucheron2013concentration}. For the sake of completeness, we present it in Appendix~\ref{dev_moment_equiv_proof:section}.
So it is quite natural to consider moment bounds and it turns out that moments are often easier to work with than deviation inequalities, for example, for lower bounds, as we will see in Section~\ref{lower_bounds}. 
We now state several well-known moment inequalities for sums and functions of independent random variables which will be our main tools. One of them is the moment version of the bounded differences inequality, which follows immediately from Theorem~15.4 in~\citep{boucheron2013concentration}.

\begin{lemma}[Bounded differences/McDiarmid's inequality]
	\label{boundeddiff}
	Consider a function $f$ of independent random variables $X_1, \dots, X_n$ that take their values in \(\mathcal{X}\). Suppose, that \(f\) satisfies the bounded differences property, namely, for any \( i = 1, \dots, n \) and any \( x_1, \dots, x_n, x_i' \in \mathcal{X} \) it holds that
	\begin{equation}
	\label{boundedcond}
	|f(x_1, \dots, x_n) - f(x_1, \dots, x_{i-1}, x_i', x_{i + 1}, \dots, x_n)| \leq \beta .
	\end{equation}
	Then, we have for any $p \ge 2$,
	\[
	\|f(X_1, \ldots, X_n) - \E f(X_1, \ldots, X_n)\|_{p} \le 2\sqrt{np}\beta \, .
	\]
\end{lemma}
Notice that it is easy to apply the above lemma in the case that \( |X_i| \leq M\) a.s. and \( \E X_i = 0 \). Since  \( \sum_{i = 1}^{n}  X_i  \) satisfies the bounded differences condition with \( \beta = 2M\), we have
\begin{equation}
\label{hoeffding}
\left\|\sum\limits_{i = 1}^n X_i \right\|_p \le 4\sqrt{np}M .
\end{equation}
We will refer to it as the moment version of \emph{Hoeffding's inequality}.

Next, we use the following version of the classical Marcinkiewicz-Zygmund inequality (we also refer to Chapter~15 in~\citep{boucheron2013concentration} that contains similar inequalities).

\begin{lemma}[Marcinkiewicz-Zygmund's inequality \citep{ren2001best}]
	\label{marcinkiewicz}
	Let \( X_1, \dots, X_n \) be \linebreak independent centered random variables with a finite \(p\)-th moment for \(p \geq 2\). Then,
	\[
	\left\| \sum_{i = 1}^{n} X_i \right\|_{p} \leq 3\sqrt{2 np} \left(\frac{1}{n} \sum_{i = 1}^{n} \| X_i \|_{p}^{p} \right)^{\frac{1}{p}} .
	\]
\end{lemma}

\section{Upper Bounds}
\label{upper_bound}
\subsection{A Moment Bound for Sums}
\label{concentration}

We present here our main result which is a moment inequality for sums of functions of $n$ independent variables.

\begin{theorem}\label{concentration_thm}
	Let $Z = (Z_{1}, \ldots, Z_n) $ be a vector of independent random variables each taking values in $\mathcal{Z}$, and let $g_{1}, \ldots, g_n$ be some functions $g_{i}: \mathcal{Z}^n \to \mathbb{R}$ such that the following holds for any $i \in [n]$:
	\begin{itemize}
		\item $\bigl|\E[g_i(Z)| Z_i]\bigr| \le M$ a.s.,
		\item $\E [g_i(Z) | Z_{[n]\setminus\{i\}}] = 0$ a.s., 
		\item $g_i$ has a bounded difference \eqref{boundedcond} $\beta$ with respect to all variables except the $i$-th variable.
	\end{itemize}
	Then, for any $p \ge 2$,
	\[
	\left\|\sum\limits_{i = 1}^n{g}_i(Z)\right\|_p \le 12 \sqrt{2} p n \beta \lceil \log_{2} n \rceil + 4 M\sqrt{pn} .
	\]
\end{theorem}

\begin{proof}
	\def\Bc{{\breve{B}}}
	\def\cond{\vert\;}
	Without loss of generality, we suppose that \( n = 2^k \). Otherwise, we can add extra functions equal to zero, increasing the number of terms by at most two times.
	
	Consider a sequence of partitions \( \mathcal{B}_{0}, \dots, \mathcal{B}_{k}\) with  \( \mathcal{B}_0 = \{  \{i\} : \; i \in [n] \} \), \( \mathcal{B}_k = \{[n]\} \), and to get \( \mathcal{B}_{l} \) from \( \mathcal{B}_{l + 1} \) we split each subset in \( \mathcal{B}_{l + 1} \) into two equal parts. We have
	\[
	\mathcal{B}_{0} = \{ \{1\}, \dots, \{2^k\}\},\ 
	\mathcal{B}_{1} = \{ \{1, 2\}, \{ 3, 4\}, \dots, \{2^k - 1, 2^k\}\},\ 
	\mathcal{B}_{k} = \{ \{1, \dots, 2^k\}\}.
	\]
	By construction, we have \( |\mathcal{B}_{l}| = 2^{k - l}\) and \( |B| = 2^l \) for each \( B \in \mathcal{B}_{l} \). 
	For each \( i \in [n] \) and \( l = 0, \dots, k \), denote by \( B^{l}(i) \in \mathcal{B}_{l} \) the only set from \(\mathcal{B}_{l}\) that contains \( i \).
	In particular, \( B^{0}(i) = \{ i \} \) and \( B^{k}(i) = [n] \). 
	
	For each \( i \in [n] \) and every \( l = 0, \dots, k \) consider the random variables
	\[
	g_{i}^{l} = g_{i}^{l}(Z_i, Z_{[n] \setminus B^{l}(i)}) = \E[g_{i} \vert\; Z_i, Z_{[n] \setminus B^{l}(i)}],
	\]
	i.e. conditioned on \(Z_i\) and all the variables that are not in the same set as \(Z_i\) in the partition \(\mathcal{B}_l\).
	In particular, \( g_{i}^{0} = g_{i} \) and \( g_{i}^{k} = \E[g_{i} \cond Z_i]\). We can write a telescopic sum for each \( i \in [n]\),
	\[
	g_i - \E[g_i \cond Z_i] = \sum_{l = 0}^{k-1} g_{i}^{l} - g_{i}^{l + 1},
	\]
	and the total sum of interest satisfies by the triangle inequality
	\begin{equation}\label{proof_triangle}
	\left\| \sum_{i = 1}^{n} g_i \right\|_{p} \leq \left\| \sum_{i = 1}^{n} \E[g_i \cond Z_i] \right\|_{p} + \sum_{l = 0}^{k-1} \left\| \sum_{i = 1}^{n} g_i^l - g_i^{l+1} \right\|_{p}. 
	\end{equation}
	Since \( |\E[g_i \cond Z_i]| \leq M \)  and $\E(\E[g_i \cond Z_i]) = 0$, by applying \eqref{hoeffding} we have
	\begin{equation}\label{proof_hoeffding}
	\left\| \sum_{i = 1}^{n} \E[g_i \cond Z_i] \right\|_{p} \leq 4 \sqrt{pn} M .
	\end{equation}
	The only non-trivial part is the second term of the r.h.s. of \eqref{proof_triangle}. Observe that 
	\[ 
	g_{i}^{l + 1}(Z_i, Z_{[n] \setminus B^{l + 1}(i)}) = \E [g_{i}^{l}(Z_i, Z_{[n] \setminus B^{l}(i)}) \cond Z_i, Z_{[n] \setminus B^{l + 1}(i)}],
	\] that is, the expectation is taken w.r.t. the variables \( Z_j, j\in B^{l + 1}(i) \setminus B^{l}(i) \). It is also not hard to see that the function \(g_i^l \) preserves the bounded differences property, just like the the function \(g_i\). Therefore, if we apply Lemma~\ref{boundeddiff} conditioned on \( Z_i, Z_{[n]\setminus B^{l+1}(i)} \), we obtain a uniform bound
	\begin{equation*}\label{g_i_subgauss}
	\| g_{i}^{l} - g_{i}^{l + 1} \|_{p}(Z_i, Z_{[n]\setminus B^{l+1}(i)}) \leq 2\sqrt{p 2^{l}} \beta,
	\qquad
	\forall p \geq 2,
	\end{equation*}
	as there are \(2^l \) indices in \(B^{l + 1}(i) \setminus B^{l}(i) \). We have as well \( \|g_{i}^{l} - g_{i}^{l + 1} \|_{p} \leq 2\sqrt{p 2^{l}} \beta \) by~\eqref{moment_cond_to_uncond}. 
	
	Let us take a look at the sum \( \sum_{i \in B^{l}} g_{i}^{l} - g_{i}^{l + 1} \) for \(B^l \in \mathcal{B}_l\). Since \( g_i^{l} - g_i^{l + 1} \) for \(i \in B^{l}\) depends only on \(Z_i, Z_{[n]\setminus B^{l}}\), the terms are independent and zero mean conditioned on \( Z_{[n]\setminus B^{l}} \).
	Applying Lemma~\ref{marcinkiewicz}, we have for any \(p \geq 2\),
	\[
	\left\| \sum_{i \in B^{l}} g_{i}^{l} - g_{i}^{l + 1} \right\|_{p}^{p}(Z_{[n]\setminus B^{l}}) \leq (3\sqrt{2p2^l})^p \frac{1}{2^l} \sum_{i \in B^l} \|g_{i}^{l} - g_{i}^{l + 1} \|_{p}^p(Z_{[n]\setminus B^{l}}).
	\]
	Integrating with respect to $(Z_{[n]\setminus B^{l}})$ and using \( \|g_{i}^{l} - g_{i}^{l + 1} \|_{p} \leq 2\sqrt{p 2^{l}} \beta \), we have
	\[
	\left\| \sum_{i \in B^{l}} g_{i}^{l} - g_{i}^{l + 1} \right\|_{p} \leq 3\sqrt{2p2^l} \times 2\sqrt{p 2^l} \beta = 6\sqrt{2} p 2^l \beta \, .
	\]
	It is left to use the triangle inequality over all sets \( B^l \in \mathcal{B}_{l} \). We have
	\[
	\left\| \sum_{i \in [n]} g_{i}^{l} - g_{i}^{l + 1} \right\|_{p}  \leq \sum_{B^l \in \mathcal{B}_l} \left\| \sum_{i \in B^{l}} g_{i}^{l} - g_{i}^{l + 1} \right\|_{p}  \leq 2^{k-l} \times 6 \sqrt{2} p 2^l \beta 
	= 6 \sqrt{2} p 2^k \beta.
	\]
	Recall, that \( 2^k < 2n \) due to the possible extension of the sample. Then
	\[
	\sum_{l = 0}^{k-1} \left\| \sum_{i = 1}^{n} g_i^l - g_i^{l+1} \right\|_{p} \leq 12 \sqrt{2} p n \beta \lceil \log_{2} n \rceil \, .
	\]
	Plugging the above bound together with \eqref{proof_hoeffding} into \eqref{proof_triangle}, we get the required bound.
\end{proof}

\begin{remark}
We remark that a version of Theorem \ref{concentration_thm} holds if e.g., $\E[g_i \cond Z_i] $ is sub-Gaussian for $i \in [n]$, but for the sake of presentation we focus on this simplified form. 
\end{remark}

\begin{remark}
	The strategy of the proof of Theorem \ref{concentration_thm} is inspired by the original approach of \citeauthor{Feldman2019}. Their \emph{clamping} can be related to the analysis of the terms \( g_{i}^{l} - g_{i}^{l + 1} \). It is important to notice that the truncation part of their analysis creates some technical difficulties since it introduces some bias and changes the stability parameter. In particular, the truncation brings an unnecessary logarithmic factor.  We entirely avoid these steps by a simple application of the Marcinkiewicz-Zygmund inequality. The analog of the \emph{dataset reduction} step of Feldman and Vondr\'ak is our nested partition scheme. However, the recursive structure of their approach is replaced by an application of telescopic sums, whereas the union bound, which also brings a logarithmic factor, is replaced by the triangle inequality for $L_p$ norms. Apart from a much shorter proof, our analysis leads to a better result: we eliminate the unnecessary $n\beta(\log n)^2$ term.  
\end{remark}

\subsection{Consequences for Uniformly Stable Algorithms}
In order to make use of Theorem \ref{concentration_thm} to obtain generalization bounds, we will consider the following functions:
\begin{equation}
\label{gi}
g_i = g_i(Z_1, \ldots, Z_n) = \E_{(X^{\prime}_i, Y^\prime_i)}\left(\E_{(X, Y)}\ell(A_{S^{i}}(X), Y) - \ell(A_{S^{i}}(X_i), Y_i)\right)\,,
\end{equation}
where we recall that
\[
S^{i} = \{(X_1, Y_1), \ldots, (X_{i - 1}, Y_{i - 1}), (X^{\prime}_i, Y^\prime_i), (X_{i + 1}, Y_{i + 1}), \dots, (X_n, Y_n)\},
\]
and $Z_i^{\prime} = (X^{\prime}_i, Y^\prime_i)$ is an independent copy of $(X_i, Y_i)$. The key observation is captured by the following simple lemma.

\begin{lemma}\label{stable_gs}
	Under the uniform stability condition with parameter $\gamma$ \eqref{stability_original} and uniform boundedness of the loss function \( \ell(\cdot, \cdot) \leq L\) we have for $g_i$ defined by $\eqref{gi}$, that
	\[
	\left||n(R(A_{S}) - R_{\text{emp}}(A_{S}))| - \left|\sum\limits_{i = 1}^ng_i\right|\right| \le 2\gamma n.
	\]
	Moreover, we have a. s. for $i \in [n]$, $|g_i| \le L$ and $\E[\E[g_i|Z_{[n]\setminus{i}}]] = 0$.
	
	Finally, as a deterministic function $g_i(z_1, \ldots, z_n)$ satisfies the bounded difference condition \eqref{boundedcond} with $\beta = 2\gamma$ for all except the $i-th$ variable.
\end{lemma}
\begin{proof}
	By uniform stability, we can write the following  decomposition 
	\begin{align*}
	&|n(R(A_{S}) - R_{\text{emp}}(A_{S}))|= \left|\sum\limits_{i = 1}^n\left(\E_{(X, Y)}\ell(A_{S}(X), Y) - \ell(A_{S}(X_i), Y_i)\right)\right| 
	\\
	&\le 2\gamma n + \left|\sum\limits_{i = 1}^n\E_{(X^{\prime}_i, Y^\prime_i)}\left(\E_{(X, Y)}\ell(A_{S^{i}}(X), Y) - \ell(A_{S^{i}}(X_i), Y_i)\right)\right| =2\gamma n + \left|\sum\limits_{i = 1}^ng_i\right|.
	\end{align*}
	Similarly to the computations above we have
	\[
	|n(R(A_{S}) - R_{\text{emp}}(A_{S}))| \ge \left|\sum\limits_{i = 1}^ng_i\right| - 2\gamma n.
	\]
	The remaining properties can be immediately verified.
\end{proof}

Using the above lemma, we can now obtain our main generalization bound.
\begin{corollary}
	\label{riskbound}
	Under the uniform stability condition \eqref{stability_original} with parameter $\gamma$ and the uniform boundedness of the loss function \( \ell(\cdot, \cdot) \leq L\),  we have that for any \(\delta \in (0, 1) \), with probability at least \(1 - \delta \),
	\[
	\label{betterrisk}
	|n(R(A_{S}) - R_{\text{emp}}(A_{S}))| \lesssim n\gamma \log n \log \left(\frac{1}{\delta}\right) + L \sqrt{n\log \left(\frac{1}{\delta}\right)} \, .
	\]
\end{corollary}
The last bound is an improvement of the recent upper bound for uniformly stable algorithms by Feldman and Vondr\'ak \eqref{secondfeldmanineq}. To be precise, we removed the unnecessary $n\gamma\log^2 n$ term.

\begin{proof}
	Combining Lemma~\ref{stable_gs} and Theorem~\ref{concentration_thm} with \( g_i\) defined in \eqref{gi}, $M = L$, and $\beta = 2\gamma$, we have for any \(p \geq 2\),
	\[
	\| n(R - R_{emp})\|_{p} \lesssim p n \gamma \log n + L \sqrt{pn} .
	\]
	The deviation bound now follows immediately from Lemma~\ref{dev_moment_equiv:lemma}.
\end{proof}

\section{Lower Bounds}
\label{lower_bounds}
Since the bound of Theorem \ref{concentration_thm} implies the best known risk bound, it is natural to ask if it can be improved in general. By Lemma~\ref{stable_gs}, we know that the analysis of the generalization bounds is closely related to the analysis of the functions satisfying the assumptions of Theorem \ref{concentration_thm}. Therefore, it is interesting to know how sharp the general bound \eqref{improvedthm} is. Recall that
\[
\left\|\sum\limits_{i = 1}^n{g}_i(Z)\right\|_p \lesssim \left(p n \beta \log n + M\sqrt{pn}\right)\wedge nL,
\]
where, as before, $L$ is a uniform bound on $|g_i|$. In this section, we prove that one can not improve the bound of Theorem \ref{concentration_thm}, apart from the $\log n$-factor, and the bound is tight with respect to the parameters \(M, \beta, n, \log \frac{1}{\delta} \) in some regimes. We notice, however, that this does not completely answer the question of the optimality of the risk bound of Corollary \ref{riskbound} for uniformly stable algorithms, but shows that this is the best we can hope for as long as our upper bound is based only on the parameters \(M, \beta, n, \log \frac{1}{\delta} \). In particular, Theorem~\ref{concentration_thm} disregards the condition \( | g_{i} | \leq L \). We discuss this in more detail in what follows.

\begin{proposition}[The lower bound, $p \le n$]
	\label{lowerbound}
	Let \(Z_1, \dots, Z_n \) be i.i.d. Rademacher signs. There is an absolute constant \( \kappa > 0 \) and functions \( g_i : \{-1, 1\}^n \rightarrow \R \) that satisfy the conditions of Theorem~\ref{concentration_thm} with the parameters \( \beta\), \(M \), such that we have for any \( \kappa \leq p \leq n\),
	\begin{equation}
	\label{lowerb}
	\left\| \sum_{i = 1}^{n} g_{i}(Z_1, \dots, Z_n) \right\|_{p} \gtrsim p n \beta + M \sqrt{pn} .
	\end{equation}
\end{proposition}

The proof is based on a moment version of the Montgomery-Smith bound (originally from \citep{montgomery1990distribution}) which is due to~\cite{hitczenko1993domination}. We postpone the proof to Appendix~\ref{lb_proof}. We note additionally that the case \(M = 0 \) follows immediately from Corollary~1, Example~2 by \cite{Latala99}.

The lower bound of Proposition~\ref{lowerbound} matches the result of Theorem~\ref{concentration_thm} up to the logarithmic factor in the regime $p \le n$. In particular, it means that in this regime, the bound has to be sub-exponential unless we use some properties of the functions \(g_i\), other than mentioned in Theorem~\ref{concentration_thm}. We additionally note that our moment lower bounds imply the deviation lower bounds. We can show that there are absolute constants \(c_1, c_2 > 0 \) such that the functions defined in \eqref{chaos} satisfy for any \( \delta \in (e^{-c_1n}, 1) \),
\begin{equation}\label{lower_deviation}
\P\left( \left| \sum_{i = 1}^{n} g_i \right| \lesssim  n \beta \log\left(\frac{1}{\delta}\right) + M \sqrt{n \log\left(\frac{1}{\delta}\right)} \right) \leq
1 - \min(c_2, \delta).
\end{equation}
This bound can be derived through the Paley-Zygmund inequality, e.g., using the standard arguments as in \citep*{gluskin1995tail}. For the sake of completeness, we derive this inequality in Appendix~\ref{paley_zygmund}.

Besides, when \( p > n\), a trivial bound \( \left\| \sum_{i = 1}^{n} g_{i} \right\|_{p} \lesssim nL \) is the best one can have. Like in \eqref{gi_lowerbound}, consider the functions \( g_{i} = L Z_i \), where \( Z_i \) are i.i.d. Rademacher signs (it corresponds to a learning algorithm that always outputs the same classifier). By Lemma~\ref{hitczenko} we have for \(p > n\),
\[
\left\| \sum_{i = 1}^{n} g_{i} \right\|_{p} = L \left\| \sum_{i = 1}^{n} Z_{i} \right\|_{p} \gtrsim Ln \, . 
\]
\section{Discussions}
\label{discussion}
\paragraph{On the Sub-Optimality of \eqref{bousquetbound}}

At first, we prove an exact moment analog of \eqref{bousquetbound} for $\left|\sum\limits_{i = 1}^ng_i\right|$, for $g_i$ defined by \eqref{gi}. We have in mind the illustrative regime of $L = 1$ and $\gamma = \frac{1}{\sqrt{n}}$, this is exactly when the bound \eqref{secondfeldmanineq} balances the two terms. By the triangle inequality we have
\begin{align*}
\left\|\sum\limits_{i = 1}^n g_i\right\|_p &\le \sum\limits_{i = 1}^n\|g_i - \E[g_i|Z_i]\|_p + \left\|\sum\limits_{i = 1}^n\E[g_i|Z_i]\right\|_p
\\
&\le 4n\sqrt{np}\gamma + 4\sqrt{pn}L.
\end{align*}
where we used that conditioned on $Z_i$ the random variable $g_i - \E[g_i|Z_i]$ is centered and combined this fact with Lemma~\ref{boundeddiff}. Since $\sum\limits_{i = 1}^n\E[g_i|Z_i]$ is a sum of independent centered bounded random variables, Hoeffding's inequality \eqref{hoeffding} is applicable to $\left\|\sum\limits_{i = 1}^n\E[g_i|Z_i]\right\|_p$.

Observe that we lose a lot by replacing $\left\|\sum\limits_{i = 1}^n(g_i - \E[g_i|Z_i])\right\|_p$ with $\sum\limits_{i = 1}^n\|g_i - \E[g_i|Z_i]\|_p$. Indeed, it is easy to see that the random variables $g_i$, as well as $g_i - \E[g_i|Z_i]$, are weakly correlated. Indeed, for $i \neq j$ using $\E[g_i|Z_{[n] \setminus \{i\}}] = 0$ and $\E[g_j|Z_{[n] \setminus \{j\}}] = 0$ together with the bounded difference property, we have
\begin{equation}
\label{corellation}
\left|\E g_i g_j\right| = |\E (g_i(Z) - g_i(Z^j))(g_j(Z) - g_j(Z^i))| \le 4\gamma^2.
\end{equation}
This suggests that for $\gamma = \frac{1}{\sqrt{n}}$, the random variables $g_i$ and $g_j$ have small correlation. However, the original argument in \citep{Bousquet02} does not take this into account and would give the same bound even if all $g_i - \E[g_i|Z_i]$ were replaced by the same random variable $g_1 - \E[g_1|Z_1]$.  

\paragraph{A Better Second Moment Bound}
Actually it turns out that using the property of weak correlation between the $g_i$ functions one can obtain a tight upper bound on the second moment.
Indeed, \eqref{corellation} with \(\gamma = \beta / 2\) together with the bounded difference property for $g_i$ readily gives
\begin{align*}
\left\|\sum_{i = 1}^n g_i\right\|_2 = &\sqrt{\sum\limits_{i \neq j}\E g_i g_j + \sum\limits_{i = 1}^n\E g^2_i} \le n\beta + \sqrt{n}\max\limits_{i \in [n]}\|g_i\|_2
\\
&\le n\beta + \sqrt{n}\max\limits_{i \in [n]}\left(\|g_i - \E [g_{i}|Z_i]\|_2 + \|\E [g_{i}|Z_i]\|_2\right)
\\
&\le (1 + 2\sqrt{2}) n\beta + \sqrt{n}M.
\end{align*}
We also note that an argument analogous to \eqref{corellation} was first used in \citep{Feldman2018} to prove the following variance bound
\begin{equation}
\label{varianceformula}
\text{Var}(n(R(A_{S}) - R_{\text{emp}}(A_{S}))) \lesssim n^2\gamma^2 + nL^2.
\end{equation}
Note that our second moment bound is a little better as it uses $M$ instead of $L$.

\paragraph{Recovering \eqref{firstfeldmanineq}}
Another interesting direction is the analysis of the first bound of Feldman and Vondr\'ak \eqref{firstfeldmanineq}, which was originally proved by the techniques taking their roots from Differential Privacy (see the discussions on various ways to prove this bound in \citep{Feldman2019}). As already noticed in \citep{Feldman2019}, the bound \eqref{firstfeldmanineq} should not be discarded due to the fact that it does not contain additional $\log n$-factors and, more importantly, it has the \emph{sub-gaussian} form, since it depends only on $\sqrt{\log \frac{1}{\delta}}$. Recall that the second bound \eqref{secondfeldmanineq} is a mixture of \emph{sub-gaussian} $\sqrt{\log \frac{1}{\delta}}$ and \emph{sub-exponential} $\log \frac{1}{\delta}$ tails. Although we can adapt the moment technique to prove \eqref{firstfeldmanineq}, we instead come to the following more general observation: 
\begin{center}
	\emph{The bound of Theorem~\ref{concentration_thm} is strong enough to almost  recover the sub-gaussian bound \eqref{firstfeldmanineq}}.
\end{center}
In order to show this, we have by Theorem~\ref{concentration_thm}, provided that $|g_i| \le L$ almost surely
\begin{equation}
\label{improvedthm}
\left\|\sum\limits_{i = 1}^n{g}_i(Z)\right\|_p \le \left(12 \sqrt{2} p n \beta \lceil \log_{2} n \rceil + 4 M\sqrt{pn}\right)\wedge nL .
\end{equation}
Since $M \le L$ and for $a, b \ge 0$, $a\wedge b \le \sqrt{ab}$ (which is rather crude) we have for \(p \leq n\),
\begin{align*}
\left\|\sum\limits_{i = 1}^n{g}_i(Z)\right\|_p & \leq  (12\sqrt{2}p n \beta \lceil \log_{2} n \rceil \wedge Ln) + 4L\sqrt{pn} \\
& \lesssim  n\sqrt{p \beta L \log n } + L\sqrt{pn}.
\end{align*}
Similarly to the proof of Corollary \ref{riskbound}, it immediately implies 
\begin{equation}
\label{secondnewbound}
n|R(A_{S}) - R_{\text{emp}}(A_{S})| \lesssim (n\sqrt{\gamma L\log n} + L\sqrt{n})\sqrt{\log \frac{1}{\delta}},
\end{equation}
which is \eqref{firstfeldmanineq} up to an unnecessary $\sqrt{\log n}$ factor. The latter is clearly an artifact of the proof in our case.

\paragraph{Interpolation Algorithms}
We observe here that our results can potentially be applied to interpolation algorithms. As mentioned earlier, in this case, since  $R_{emp}(A_S)=0$ one cannot hope to get any non-trivial bound on the generalization error (since there are distributions for which $\E R(A_S)$ cannot go to zero).
The idea is then to use another estimator of the risk, namely the leave-one-out estimator, which can be defined as
\[
R_{loo}(A_S)=\frac{1}{n}\sum_{i=1}^n \ell\left(A_{S^{\setminus i}}(X_i),Y_i\right)\,, 
\]
where $S^{\setminus i}$ is $S\setminus \{(X_i, Y_i)\}$. Then an analogue of Lemma \ref{stable_gs} can be obtained with $R_{loo}$ instead of $R_{emp}$.
Note that one also needs to modify the notion of uniform stability to exclude the case where the test point is in the training set. However, proving non-trivial stability results for realistic interpolation algorithms or extending our results to use something like hypothesis stability (which is an in-expectation stability instead of the worst-case one) would still require more work.

\section{Open Questions}
We have presented some progress towards narrowing down the exact behaviour of the difference between true and empirical risk for uniformly stable algorithms with almost matching upper and lower bounds. But there are a few remaining open questions regarding the tightness of our results. For the upper bound,
since it is possible to get rid of the factor $\log n$ in the second moment, as demonstrated by \eqref{varianceformula}, there is hope that this factor can be completely removed for $p > 2$ in Theorem \ref{concentration_thm} as well as in \eqref{bestbound}.

For the lower bound, it would be interesting to find a learning algorithm with uniformly bounded functions \( |g_{i} | \leq L \) and a generalization bound matching the high probability lower bound~\eqref{lower_deviation}. Indeed, Proposition \ref{lowerbound} shows that the high probability upper bound of Theorem \ref{concentration_thm} can not be improved in general (apart from the logarithmic factor). At the same time, our bound does not take into account that $L$ can actually be of the same order as $M$, which happens in the context of uniformly stable learning algorithms. In particular, the example of Proposition~\eqref{lower_deviation} has $L = \beta(n - 1)$, which can be much larger than $M$ in the regime of $\beta = \frac{1}{\sqrt{n}}$. Therefore, in this context it will be interesting to find (if it is possible at all) a lower bound of the form
\[
\left\| \sum_{i = 1}^{n} g_{i}(Z_1, \dots, Z_n) \right\|_{p} \gtrsim p n \beta + L \sqrt{pn}, 
\]
for the functions $g_i$ corresponding to a uniformly stable algorithm.

Finally, some extensions of Lemma \ref{boundeddiff} are known in the unbounded case (see e.g., \citep*{kontorovich2014concentration}). This may  lead to the generalization bounds for the unbounded loss functions using our Theorem \ref{concentration_thm}. Following \cite{Feldman2019}, there is a renewed interest in the use of stability bounds for studying learning algorithms, for example \citep*{foster2019hypothesis}, whose results can possibly be further improved using the techniques we presented.

\acks{We would like to thank Vitaly Feldman and Shay Moran for stimulating discussions and Jan Vondr\'ak for valuable feedback.}

\bibliography{mybib}

\appendix

\section{Proof of Lemma~\ref{dev_moment_equiv:lemma}}\label{dev_moment_equiv_proof:section}

It is easy to see that the random variable \( (|Y| - a - b)_{+} \) satisfies the requirements of Theorem~2.3 in \citep{boucheron2013concentration}. Therefore, we have for each integer \(p \geq 1\),
\begin{align*}
\| (|Y| - a - b)_{+} \|_{2p} & \leq \sqrt[2p]{p! (2a)^{2p} + (2p)! (4b)^{2p}} \\
& \leq \sqrt{2 p} a + 4 pb ,
\end{align*}
and the first bound follows from the trivial \( \| Y \|_{p} \leq \| (|Y| - a - b)_{+} \|_{2p} + a + b \). 

On the other hand, if \( \| Y \|_{p} \leq \sqrt{p} a + pb \) for any \(p \geq 1\), by Markov's inequality for any \( \delta \in (0, 1)\),
\[
\P\left( |Y| > \| Y \|_{p} e^{\frac{\log\left(\frac{1}{\delta}\right)}{p}} \right) \leq \left( \frac{\| Y \|_{p}}{\| Y \|_{p} e^{\frac{\log\left(\frac{1}{\delta}\right)}{p}}}\right)^{p} = \delta.
\]
It remains to pick \( p = \log \left(\frac{e}{\delta} \right) \geq 1\), so that the exponent disappears and we get the required bound.

\section{Proof of Proposition~\ref{lowerbound}}
\label{lb_proof}

As before, we need two well-known facts from probability theory. The first lemma is a moment version of the Montgomery-Smith bound (originally from \citep{montgomery1990distribution}) which is due to~\cite{hitczenko1993domination}. It characterizes the moments of Rademacher sums up to a multiplicative constant factor. 

\begin{lemma}[Moments of weighted Rademacher sums, \citep{hitczenko1993domination}]
	\label{hitczenko}
	Let \(a_1 \geq  \ldots \geq a_n\) be a non-increasing sequence of non-negative numbers and let \(\varepsilon_1, \dots, \varepsilon_n \) denote i.i.d. Rademacher signs. Then,
	\[
	\left\| \sum_{i = 1}^{n} a_i \varepsilon_i \right\|_{p} \sim \sum_{i = 1}^{p} a_i + \sqrt{p} \left( \sum_{i = p + 1}^{n} a_i^2 \right)^{1/2} . 
	\]
\end{lemma}
The next lemma is Chebyshev's association  inequality, see e.g., Theorem~{2.14} in \citep{boucheron2013concentration}.
\begin{lemma}
	\label{Chebyshev}
	Let $f$ and $h$ be non-decreasing real-valued functions defined on the real line. If $X$ is a real-valued
	random variable, then
	\[
	\E g(X)h(X) \ge \E g(X)\E h(X) .
	\]
\end{lemma}

Consider the functions
\begin{equation}
\label{chaos}
g_i = g_i(Z_1, \ldots, Z_n) = MZ_i + \frac{\beta}{2} Z_i\left(\sum\limits_{j \neq i}Z_i\right).
\end{equation}
It is easy to check that \( \E[g_i \cond Z_{[n]\setminus\{i\}}] = 0 \) and \( \E[g_i \cond Z_i] = M \) a.s. Moreover, each \(g_i\) satisfies the bounded difference property with parameter \(\beta\) w.r.t. all except the \(i\)-th variable. Denoting \( S = \sum\limits_{i = 1}^{n} Z_i \) we have
\begin{equation}\label{gi_lowerbound}
\sum_{i = 1}^{n} g_{i} = MS + {\frac{\beta}{2}} S^{2} - {\frac{\beta}{2}} n.
\end{equation}
By the triangle inequality we have, \[
\left\| \sum_{i \leq n} g_{i} \right\|_{p} \geq \left\| MS + {\frac{\beta}{2}} S^{2} \right\|_{p} - {\frac{\beta}{2}} n.
\]
For \( S_+ = S \Ind[S \geq 0] \) we obviously have \(  \| MS + \beta S^{2} / 2 \|_{p} \geq \| MS_{+} + \beta S_{+}^{2} / 2 \|_{p} \).
By Lemma~\ref{Chebyshev} and since both functions \(x^{p}\) and \((M + \beta x / 2)^{p} \) are non-decreasing for non-negative \(x\), we have
\[
\| MS_{+} + \beta S_{+}^{2}/2 \|_{p} \geq \| S_{+} \|_{p} \| M + \beta S_{+}/2 \|_{p} \geq \| S_{+} \|_{p} (M \vee \beta \|S_{+}\|_{p}/2)\, .
\]
Finally, due to the symmetry of \(S \) and Lemma~\ref{hitczenko}, we have for \(p \leq n\),
\[
\| S_{+} \|_{p} \geq \frac{1}{2} \| S \|_{p} \gtrsim \sqrt{pn} .
\]
Therefore, for some $c> 0$, our construction implies the following lower bound
\[
\left\|\sum\limits_{i = 1}^n g_i\right\|_p \geq c \sqrt{pn}(M + \beta \sqrt{pn}/2) - n\beta/2 = (cp - 1) \beta n/2 + \sqrt{pn} M .
\]

\section{Proof of the Lower Tail \texorpdfstring{\eqref{lower_deviation}}{Lg}}
\label{paley_zygmund}
By the Paley-Zygmund inequality (see \citep{gluskin1995tail}) we have for \( f = \sum_{i = 1}^{n} g_{i} \), where $g_i$ are defined by \eqref{chaos},
\begin{align*}
\P\left( |f| \geq \frac{1}{2} \|f\|_{p}  \right) &= \P\left(|f|^{p} \geq \frac{1}{2^p} \E |f|^{p} \right) \geq \left( 1 - \frac{1}{2^p} \right) \frac{(\E |f|^{p})^{2}}{\E |f|^{2p}} \\
&\geq
\left( 1 - \frac{1}{2^p} \right) \left( \frac{\| f\|_{p}}{\|f\|_{2p}}\right)^{2p} \geq \left( \frac{\| f \|_{p}^{2}}{2\| f\|_{2p}^2}\right)^{p}.
\end{align*}
In Proposition~\ref{lowerbound}, we derived a lower bound \(  \| f\|_{p} \gtrsim pn\beta + M \sqrt{pn}  \) for \( \kappa \leq p \leq n \). In our case, it is also not hard to get a matching upper bound without the logarithm via Lata\l{}a's bound for Rademacher chaos. We note that if we are only interested in the upper bound, we can alternatively use a moment version of the Hanson-Wright inequality (see e.g., \citep{Vershynin2016HDP}).
\begin{lemma}[Corollary 1 and Example 2 in~\citep{Latala99}]
	\label{latalaslemma}
	Let $(a_{i, j})$ be an $n \times n$ real symmetric matrix with $a_{i,i} = 0$ for all $i \in [n]$ and let \(\varepsilon_1, \dots, \varepsilon_n \) denote i.i.d. Rademacher signs. We have that for any $p \ge 1$,
	\begin{align*}
	\left\|\sum\limits_{i, j}a_{i, j}\varepsilon_i\varepsilon_j\right\|_p &\sim \sup\left\{\sum\limits_{i, j}a_{i, j}x_i y_j: \|x\|_2, \|y\|_2 \le \sqrt{p}, \|x\|_{\infty}, \|y\|_{\infty} \le 1\right\} 
	\\
	&\quad\quad+\sum\limits_{i = 1}^p A_i^*+ \sqrt{p}\left(\sum\limits_{i > p}^n (A_i^*)^2\right)^{1/2},
	\end{align*}
	where $A_i^*$ denotes the non-decreasing rearrangement of the sequence $A_i = \left(\sum\limits_{j}^na_{i,j}^2\right)^{1/2}$.
\end{lemma}
Using Lemma \ref{latalaslemma} we can easily control $\left\| \sum_{i \neq j} Z_i Z_j \right\|_{p}$ for $p \le n$. Observe that in our case $a_{i, j} = 1$ for $i \neq j$ and $A_i = \sqrt{n - 1}$ for $i\in [n]$. It implies that for any \( p \ge 1 \),
\begin{equation}
\label{quadraticform}
\left\| \sum_{i \neq j} Z_i Z_j \right\|_{p} \lesssim  pn + p\sqrt{n} + \sqrt{p}n,
\end{equation}
where for the first term in \eqref{quadraticform} we used $\sup\left\{\sum\limits_{i \neq j}x_i y_j: \|x\|_2, \|y\|_2 \le \sqrt{p}\right\} \le pn$. Finally, we have for any \( p \ge 1 \),
\[
\| f \|_{p} \leq M \left\| \sum_{i = 1}^{n} Z_i \right\|_{p} + \frac{\beta}{2} \left\| \sum_{i \neq j} Z_i Z_j \right\|_{p} \lesssim M \sqrt{pn} + p n \beta .
\]
Since \( \frac{a \sqrt{p} + b p}{a \sqrt{2p} + 2bp} \geq \frac{1}{2} \), we have that for any $p$ such that \( \kappa \leq p \leq n / 2 \), it holds that \( \| f\|_{p} \geq c \| f \|_{2p} \), where \( c \leq 1 \) is an absolute constant. Therefore,
\[
\P \left( |f| \geq \frac{1}{2} \| f\|_{p} \right) \geq \left(\frac{c^2}{2}\right)^{p}.
\]
Let us pick \( p = \kappa \vee p_{\delta}\), where \(p_{\delta} = (\log 2c^{-2})^{-1} \log\left( \frac{1}{\delta} \right)\), so that \(  (c^2 / 2)^{p_{\delta}} = \delta \). Assuming that \( \delta \geq e^{-(\log 2c^{-2})^{-1} n} \), we have \(p_{\delta} \leq n \). Since \( \kappa \leq p \leq n \) and \(c^2 / 2 < 1\), we have
\[
\P\left( |f| \geq \| f \|_{p_\delta} \right) \geq \P\left( | f| \geq \| f\|_{p} \right) \geq \left(\frac{c^2}{2}\right)^{\kappa} \wedge \delta.
\]
It is left to notice that our lower bound implies
\[
\| f \|_{p_{\delta}} \gtrsim n\beta \log\left(\frac{1}{\delta}\right) + M \sqrt{n \log\left(\frac{1}{\delta}\right)} \, .
\]
Therefore, \eqref{lower_deviation} holds for \(c_1 = (\log 2c^{-2})^{-1}\) and \(c_2 = (c^2/2)^{\kappa}\).

\end{document}